\documentclass{article}

 \usepackage[final]{neurips_2021}

\usepackage[utf8]{inputenc} 
\usepackage[T1]{fontenc}    
\usepackage{hyperref}       
\usepackage{url}            
\usepackage{booktabs}       
\usepackage{amsfonts}       
\usepackage{nicefrac}       
\usepackage{microtype}      
\usepackage{xcolor}         

\usepackage{authblk}

\usepackage{graphicx}
\usepackage[percent]{overpic}
\usepackage[
  font = small,
  labelfont = bf,
  tableposition = top
]{caption}
\usepackage{subcaption}
\usepackage{multirow}
\usepackage{amsmath}
\DeclareMathOperator*{\argmax}{arg\,max}
\DeclareMathOperator*{\argmin}{arg\,min}
\usepackage{bbm}
\usepackage{amsthm}
\usepackage[numbers]{natbib}
\newtheorem{theorem}{Theorem}[section]
\DeclareMathOperator{\EX}{\mathbb{E}}

\title{Emergent Discrete Communication in Semantic Spaces}

\author{%
  Mycal Tucker$^1$, Huao Li$^2$, Siddharth Agrawal$^3$, Dana Hughes$^3$, Katia Sycara$^3$, Michael Lewis$^2$, and Julie Shah$^1$ \\
  $1$ Massachusetts Institute of Technology\\
  $2$ University of Pittsburgh\\
  $3$ Carnegie Mellon University\\
  \texttt{mycal@mit.edu, hul52@upitt.edu, siddhara@cs.cmu.edu} \\
}

\begin{document}

\maketitle

\begin{abstract}
Neural agents trained in reinforcement learning settings can learn to communicate among themselves via discrete tokens, accomplishing as a team what agents would be unable to do alone.
However, the current standard of using one-hot vectors as discrete communication tokens prevents agents from acquiring more desirable aspects of communication such as zero-shot understanding.
Inspired by word embedding techniques from natural language processing, we propose neural agent architectures that enables them to communicate via discrete tokens derived from a learned, continuous space.
We show in a decision theoretic framework that our technique optimizes communication over a wide range of scenarios, whereas one-hot tokens are only optimal under restrictive assumptions.
In self-play experiments, we validate that our trained agents learn to cluster tokens in semantically-meaningful ways, allowing them communicate in noisy environments where other techniques fail.
Lastly, we demonstrate both that agents using our method can effectively respond to novel human communication and that humans can understand unlabeled emergent agent communication, outperforming the use of one-hot communication.
\end{abstract}

\section{Introduction}
A longstanding goal of AI has been to develop agents that can cooperate with other agents or humans to accomplish tasks together.
Often, communication is necessary to enable such cooperation; the study of emergent communication has recently shown great success in producing agents that learn to communicate.
In reinforcement learning settings, guided only by environment reward, neural agents can learn to communicate by broadcasting numerical vectors to each other \citep{foerster2016learning,havrylov2017emergence,mordatch2017emergence,lazaridou2020emergent}.

Given this success, and in part inspired by the discrete nature of words in natural language, some researchers  have focused on emergent discrete communication by forcing agents to broadcast one-hot vectors \citep{foerster2016learning,lazaridou2016multi}.
These tokens in effect become a lexicon used by agents.
Studying when these tokens are emitted allows researchers to uncover their meanings, as well as to study the broader questions of what environment or agent factors contribute to desirable aspects of learned communication (e.g., compositionality or, in continuous communication settings, zero-shot understanding) \citep{kottur2017natural,lazaridou2018emergence,bullard2020exploring,bullard2021quasi}.

We claim that discretizing messages by constraining them to conform to one-hot vectors fundamentally precludes agents from learning some desirable properties of language.
One-hot vectors establish no relationships between tokens because each one-hot vector is orthogonal to and equally far away from all other vectors.
Conversely, research from natural language processing and word embeddings has long established the importance of learning representations of discrete words within a continuous, semantic space \citep{word2vec,glove}.

In this work, we demonstrate the benefit of agents that employ a discrete set of tokens within a continuous space over agents that use the standard practice of communicating via one-hot vectors in discrete emergent communication settings.
We present a novel architecture and implementation for learning such communication and provide decision-theoretic analysis of the value of such an approach - the congruence of meaning and form of communications.
Simulation experiments confirmed these results: our agents learned an arrangement of tokens that clustered in human-understandable patterns.
The arrangement of discrete tokens within the learned communication space produced team performance that was robust to environment noise and enabled agents to effectively utilize novel communication vectors.
In human-agent experiments, agents aligned their tokens with natural language embeddings and responded appropriately to novel English phrases.
Lastly, we showed that humans capably interpreted unlabeled emergent communication tokens in a reference game.\footnote{Anonymized code available at \href{https://anonymous.4open.science/r/NeurIPS-protocomms/README.md}{https://anonymous.4open.science/r/NeurIPS-protocomms}}

\section{Related Work}
\label{related}
We propose a technique within emergent communication literature, drawing inspiration from work on word embeddings in natural language processing (NLP) and zero-shot classification.

\subsection{Emergent Communication}
Researchers of emergent communication study techniques to enable agents to learn to communicate among themselves, enabling high task performance in reinforcement learning settings (see  \cite{sukhbaatar2016learning,hausknecht2016cooperation,lazaridou2020emergent}, among others).
These settings, such as reference games or Lewis signalling games \citep{lewis2008convention}, are designed such that agents must communicate to perform the task successfully; in ``cheap-talk'' scenarios, agents often learn successful communication strategies by sending real-valued vectors to each other.
We focus on discrete communication emerging among decentralized agents that may communicate by sending one of a finite set of vectors to each other, but may not access other agents' weights or gradients during training or execution \citep{NEURIPS2019_fe5e7cb6}.
Previously, such discrete communication has often taken the form of one-hot messages: Foerster et al. \cite{foerster2016learning} proposed binary discrete messages, but subsequent works seem to have reverted to one-hot vectors \citep{kottur2017natural,lazaridou2016multi,havrylov2017emergence,lee2017emergent,chaabouni2019antiefficient}.

Even when agents learn to communicate, they often fail to learn a protocol that humans or separately-trained agents can understand; that is, they fail at the ``zero-shot'' learning problem \citep{hu2020other}.
To address this gap, several recent works have found properties of environments that encourage zero-shot continuous communication \citep{bullard2020exploring,bullard2021quasi}.
Other research addresses human-understandable communication in RL settings.
Using bottom-up approaches, researchers align emergent communication tokens with human language \citep{lee2019countering,lowe2020interaction};  Lazaridou et al. \cite{lazaridou2016multi} even demonstrated that humans can understand agents' tokens in zero-shot settings (e.g., interpreting the agent's token for ``dolphin'' as referring to a photo of water).
Complementary, ``top-down'' approaches leverage pre-trained task-specific language models for natural language communication \citep{lazaridou-etal-2020-multi}.

We position our work within the field of emergent discrete communication: we seek to have agents learn to communicate via a finite set of possible messages, without access to a task-specific language model.
Unlike prior work, we relax the assumption that discrete tokens must take the form of one-hot vectors and demonstrate that our proposed architecture enables agent utilization of novel communication, and supports human understanding of learned communications.





\subsection{Natural Language Word Embeddings}
While early neural NLP research used one-hot representations of words, most modern techniques rely upon dense word embeddings such as Word2Vec or GLoVE for faster training and greater generalization power \citep{bengio2003neural,word2vec,glove,almeida2019word}.
In essence, these embedding techniques learn discrete representations of words within a semantically meaningful space; words representing similar meanings are often embedded in similar locations.
We argue that current research in emergent discrete communication that relies upon one-hot tokens is similar to the early one-hot word embedding techniques that suffered from limited generalizability.
Instead, we propose that discrete emergent communication should be learned within a continuous, semantically meaningful space, like modern word embeddings.

\subsection{Zero-Shot Classification}
Research in zero and few-shot classification seeks to train classifiers that are able to correctly classify inputs belonging to new classes despite seeing few or no training examples of that class \citep{xian2018zero,meng2020unsupervised,chen2018zero}.
Often, techniques employ ``side-information'' to enable high performance, for example using labeled data from a different domain \citep{meng2020unsupervised} or information from language \citep{socher2013zero}.
We focus on emergent communication rather than classification, but in experiments aligning communication tokens with word embeddings, we similarly exploit side information for zero-shot understanding.

\section{Technical Approach}
\subsection{Emergent Communication in Multi Agent Reinforcement Learning}
We adopt standard methods for training multiple agents in a decentralized partially-observable Markov decision process \citep{bernstein2002complexity}.
The Dec-POMDP is defined by the $(S, A, T, R, O, \Omega, \gamma)$ tuple.
$S$ is the set of states; $A_i \forall i \in [1, N]$ are the sets of actions, including communication, for each of $N$ agents; $T: S \times A_1 ... A_N \longrightarrow S$ is the probabilistic transition between states due to joint actions.
We focus on partially-observable settings to encourage communication among agents.
$\Omega$ defines the set of possible observations; $O_i: A_1 ... A_N \times S \longrightarrow \Omega$ maps from joint actions and the state to distributions of observations for each agent.
Lastly, $\gamma$ and $R$ define the discount factor and reward function, respectively.
The goal is to find the policies for all agents, $\pi_i: o \in \Omega \longrightarrow A_i$, that maximize the expected discounted reward.
Multi-agent Dec-POMDPs are well studied \citep{mnih2016asynchronous,lowe2017multi,schulman2017proximal}; we borrow existing techniques for solving these problems in multi-agent reinforcement learning settings, using an agent architecture we develop. 

\subsection{Decision-Theoretic Noisy Discrete Communication}
\label{sec:analysis}
In this section, we formalize a noisy channel reference game, an instance of a Dec-POMDP.
We derive that one-hot vectors are optimal tokens under restrictive assumptions, and relaxation of those assumptions may lead to suboptimality of one-hot communication.

Consider two agents, a speaker and a listener, that communicate via a $c$ dimensional noisy channel,
Each episode is composed of two timesteps.
In the first timestep, a speaker emits a token, $t$, chosen from a distribution over columns in a matrix of $z$ tokens, $T$; $t \in T_{c \times z}$.
In the second timestep, the listener observes a message, $m \in R^{c}$, a noisy version of $t$, corrupted by additive zero-mean noise.
We constrain all tokens to fall within the unit hypercube (i.e., $0 \leq T[i, j] \leq 1; i \in [0, c], j \in [0, z]$), which constrains the relative scales of tokens and noise.

The listener's task is to predict which token the speaker emitted, given the observed noisy message.
We assume there exists some reward matrix, $R_{z \times z}$ that specifies the shared reward among agents, where entry $R[i, j]$ corresponds to the reward if the speaker emits $T[i]$ and the listener predicts $T[j]$; standard reward matrices resemble the identity matrix, potentially with some small, positive values off the diagonal for ``reasonable'' mistakes.
From a decision-theoretic framework, the speaker and listener wish to maximize expected reward.
We seek to find the optimal set of tokens, $T^*$, that maximizes this expected value, as shown below:


\begin{alignat}{2}
    T^* & = \argmax_T \sum_{i \in z} P(T[i])\underset{P(m_i | T[i])}{\EX}  \left[ \max_{j \in z} R[i, j] \frac{P(T[j] | m_i) P(T[j])}{\sum_{k \in z} P(T[k] | m_i) P(T[k])} \right]\\
    &= \argmax_T \sum_{i \in z} \EX_{P(m_i | T[i])} \left[ \frac{P(T[i] | m_i)}{\sum_{k \in z} P(T[k] | m_i)} \right] \hskip0.13\textwidth  \mbox{Uniform prior;}\quad R = I\\
      & = \argmin_T \sum_{i \in z} \EX_{P(m_i | T[i])} \left[ \sum_{j \in z; j\neq i} P(T[j] | m_i) \right] \hskip0.18\textwidth \mbox{}\\
      &= \argmax_T \sum_{i \in z} \sum_{j \in z; j \neq i} || T[i] - T[j] ||^2  \hskip0.33\textwidth \mbox{Gaussian noise}
\end{alignat}


Equation 1 specifies the expected reward using a rational listener model: taking the expectation over tokens (the outer sum), the listener chooses the token $T[j]$ that maximizes the expected reward given the observed message, $m_i$.
Equation 2 follows under assumptions of $z = c$ (an equal number of tokens and communication dimensions), $R = I_{z \times z}$, and a uniform prior over emitted tokens.
Equation 3 follows, deriving that the optimal tokens minimize the likelihood of confusion of any pair of tokens.
Finally, assuming a Gaussian noise model with a variance matrix equal to $\sigma I$ for some constant $\sigma$, Equation 4 states that the optimal tokens maximize the mean squared euclidean distance between tokens.
If tokens are constrained within a unit hypercube, one-hot encodings are an optimal solution.
A proof of optimality of one-hot tokens is included in Appendix~\ref{app:proof}.

While one-hot encodings may be optimal in some scenarios, the above analysis also reveals when they may not be: when the cost of errors or prior over tokens are not uniform, or when the number of tokens and communication dimensions are not equal.
Examples of suboptimality of one-hot tokens when these assumptions are relaxed are included in Appendix~\ref{app:proof}.
Thus, although many academic experiments respect these assumptions, a more powerful framework for learning tokens is needed in the general case.

\subsection{Discrete Prototype Communication}
\label{sec:method}
We present our method for prototype-based communication to address the limitations of one-hot communication by endowing agents with learnable tokens.
We create neural agents with a communication head instantiated as a multi-layer perceptron with a penultimate softmax layer (where we use the gumbel softmax trick \citep{jang2017categorical}), which then is multiplied by a learnable matrix, $T_{z \times c}$, for the final output.
More formally, an agent, parametrized by weights $\Theta$ and $T$, produces a communication token $t$ according to Equation~\ref{eq:agent_weights}.

\begin{equation}
\label{eq:agent_weights}
    \begin{split}
    \pi & = f_{\theta}(O) \\
    d &= \texttt{one\_hot}(\argmax_{i} [g_i + \log \pi_i])\\
    t & = d \times T_{z \times c}\\
    \end{split}
\end{equation}

Equation~\ref{eq:agent_weights} states that agents produce logits, $\pi$, from an observation, $O$, convert the logits to a one-hot vector, $d$, by using sampled $g_i$ drawn i.i.d. from a Gumbel(0, 1) distribution, and then multiply $d$ by $T$ to produce the token.\footnote{In experiments, we constrained the tokens to the unit hypercube by passing $t$ through a sigmoid activation to enable fair comparisons to one-hot encodings, but tokens need not generally obey such constraints.}
Although the argmax operation is not differentiable, we use it in the forward pass (that is, actually using ``hard'' argmax) but use the gumbel softmax operation for calculating gradients, used in the backwards pass.
We refer readers to Jang et al. \cite{jang2017categorical} for more details on the gumbel softmax approximation.

We note that our technique introduces an additional set of learnable parameters compared to one-hot communicative agents, which would merely output $d$.
This increased flexibility may enable improved performance but may also increase learning difficulty; experiment results, reported in Sections~\ref{sec:self} and \ref{sec:human}, indicated net benefits derived from our technique.

\subsection{Opportunities for Behavior Shaping}
\label{sec:shaping}
The previous section illustrates how $T$ may be learned using traditional reinforcement learning techniques; using a policy gradient method, the tokens will evolve during training to enable high task performance.
While this emergent behavior is desirable for its flexibility, we also examined two ways of constraining learned behavior: ``lexicon-setting'' and grounding.

The first constraint, which we call ``lexicon-setting,'' consisted of manually specifying a non-trainable $T$.
During training, gradients passed through $T$ but did not update its elements.
This allowed a designer to specify the set of tokens agents used to communicate, in other words setting the ``lexicon.''

The second constraint consisted of teaching grounding via a small set of supervised data.
Even with hand-specified tokens, it is unclear \textit{a priori} what meanings agents will associate with tokens.
To enable human-agent communication, we require a common understanding, or grounding, of tokens.
Inspired by prior art in learning social conventions in RL settings \citep{lerer2019learning,tucker2020adversarially,lowe2020interaction}, we used a small set of data mapping observations to desired communications.
The gradient of the supervised loss (mean squared error of predicted communication vs. desired communication) was added to the policy loss from the RL environment, leading the agent to use communication that matched the grounding data and enabled high task performance.

\section{Experiment Preliminaries}
\subsection{Baselines}
\label{sec:baselines}
We compared our technique to two emergent communication baselines: continuous and one-hot communication.
The continuous communication method (henceforth, \textit{cont}) allowed agents to output real-valued vectors in the communication space, $R^c$, bounded along each dimension within $(-1, 1)$ by using a \texttt{tanh} activation.
The discrete one-hot baseline (henceforth, \textit{one-hot}) allowed agents to output a one-hot vector within the communication space using the gumbel softmax trick.
By definition, this constrained \textit{one-hot} agents to output one of $c$ tokens.

All techniques were trained using the multi-agent deep deterministic policy gradient (MADDPG) method proposed by Lowe et al. \cite{lowe2017multi}, a common policy-gradient method that updates neural network weights to maximize the expected discounted reward.


\subsection{Hypotheses}
Throughout our experiments, we assessed trained agents by measuring environment-specific reward under different scenarios.
First, we measured self-play reward for agents trained together, as the environment reward they received during 500 evaluation episodes.
Second, we measured aspects of zero-shot behavior and performance in a reference game (introduced in Section~\ref{sec:human}).
Measurements included a listener agent's accuracy in selecting a target image when using novel communication generated by humans, as well as human participant accuracy in identifying a target image based on unlabeled agent communication. 

Using these metrics, we formulated the following hypotheses:
\begin{itemize}
    \item \textbf{H1:} In noisy environments, prototype-based communication will enable agents to achieve higher self-play reward than one-hot tokens.
    \item \textbf{H2:} Prototype-based agents, trained with word embeddings as tokens, will identify target images in reference games based on novel, human-generated communication, whereas one-hot encodings will fail to outperform random chance.
    \item \textbf{H3:} Human participants will accurately identify target images in reference games based on tokens learned by prototype-based agents.
\end{itemize}

Our hypotheses were informed by prototype-based communication design, which we expected to cluster tokens with similar meanings, thereby providing robustness in noisy environments and benefiting human understanding of the communication space.
We expected this improved semantic understanding of the communication space to bolster zero-shot understanding of novel communication.

\begin{figure}
    \centering
    \begin{subfigure}[b]{0.32\textwidth}
        \centering
        \includegraphics[width=\textwidth]{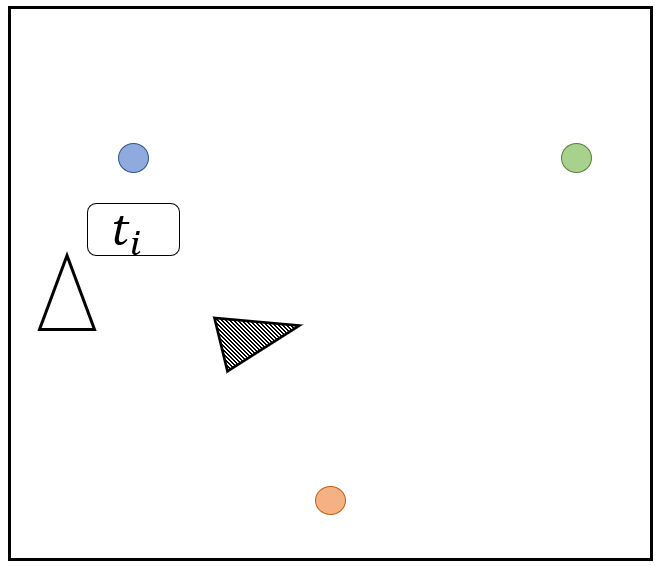}
        \caption{Triangle}
    \end{subfigure}
    \hfill
    \begin{subfigure}[b]{0.32\textwidth}
        \centering
        \includegraphics[width=\textwidth]{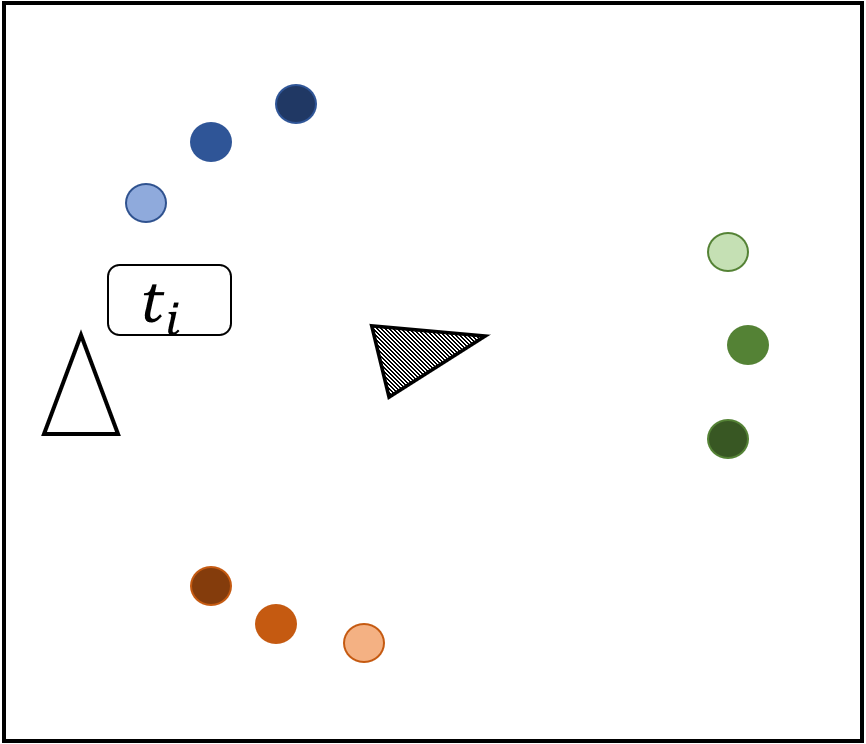}
        \caption{9 Points}
    \end{subfigure}
    \hfill
    \begin{subfigure}[b]{0.32\textwidth}
        \centering
        \includegraphics[width=\textwidth]{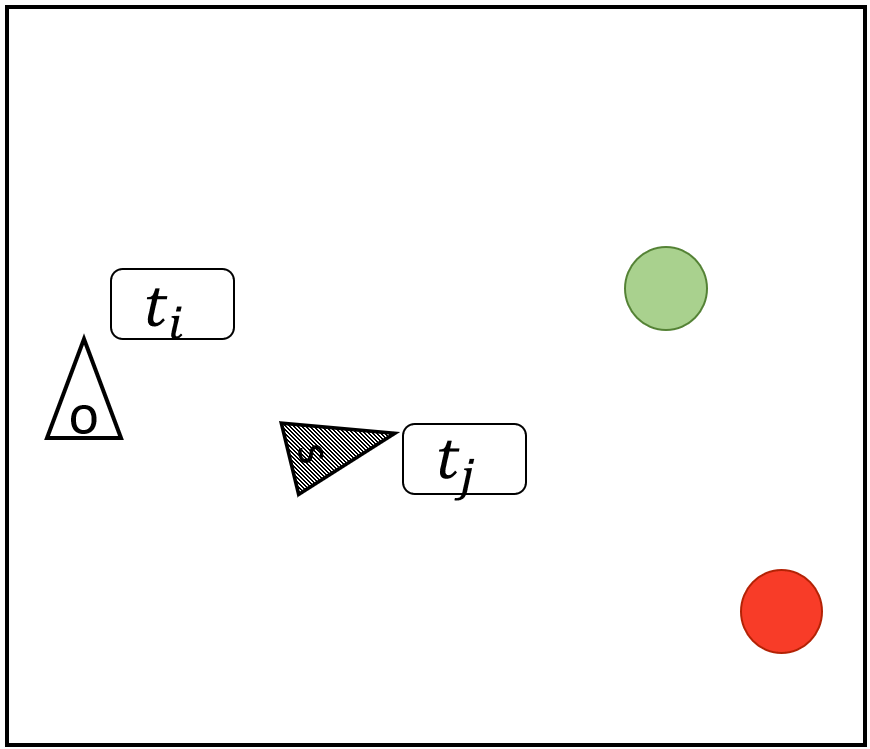}
        \caption{Uniform}
    \end{subfigure}
    \caption{The particle environments used for agent-only testing. In the triangle and 9 Points environments, the landmarks (colored circles) were in fixed locations; a target was chosen among the possible landmarks. In the uniform environment, the landmarks were in locations chosen uniformly at random within the game.}
    \label{fig:particle_worlds}
\end{figure}

\section{Agent Self-Play Experiments}
\label{sec:self}
We first evaluated our technique in particle-world environments, inspired by the environments introduced by Lowe et al. \cite{lowe2017multi} and Mordatch and Abbeel \cite{mordatch2017emergence}.
Agents were trained and tested in fixed teams; thus, these experiments were used to measure robustness to noise and assess characteristics of learned tokens.
Depictions of our environments are shown in Figure~\ref{fig:particle_worlds}.
Full details of the environments are included in Appendix~\ref{app:envs}.

In the first environment, \textit{triangle}, speaker and listener agents were spawned at the origin, and three landmarks were located at the points of an equilateral triangle.
In each episode, a landmark was designated as the target, and noise was drawn from a $c = 9$ dimensional zero-mean Gaussian with $\Sigma = 0.9 I_{c}$.
Only the speaker observed the position of the target; thus, to maximize the reward (negative distance from the listener to the target), the speaker must communicate to the listener, via a 9D vector, and the listener must move to the target.

We investigated the effect of different priors over tokens by changing the distribution over which landmark was chosen as the target: in one setting, targets were drawn uniformly, but in the other, targets were drawn from a skewed categorical distribution of $[0.49, 0.49, 0.02]$.
Three learnable prototypes were defined for our technique ($z = 3$).
Note that this meant that our prototype technique had fewer tokens than the 9 for \textit{one-hot}: we selected $c = 9$ to study scenarios in which $z < c$.

In our second environment, \textit{9 points}, speaker and listener agents were spawned at the origin, with the speaker observing the location of the target and the listener observing its own state.
Reward was similarly defined as the negative distance squared from the listener to the target.
9 landmarks were spawned in fixed locations, depicted in Figure~\ref{fig:particle_worlds}b.
This arrangement of landmarks produced a non-identity reward matrix: the listener confusing two nearby landmarks was less costly than confusing two landmarks that were further apart.
We used the \textit{9 points} environment in 3 settings: two settings with $\sigma = 0$ and $c = 3$ or $c = 9$, and one noisy setting with $\sigma = 0.5$ and $c = 9$.
In \textit{9 points}, our technique used 9 prototypes.
These settings allowed us to study cases when $z \geq c$.
Results from the first two environments, assessing robustness to noise, are shown in Table~\ref{tab:robustness}.

\begin{table}
	\begin{minipage}{0.5\linewidth}
		\caption{Median (standard error) self-play reward over 5 runs. Enabling learnable tokens endowed our agents with greater robustness to environmental noise.}
		\label{tab:robustness}
		\centering
    \begin{tabular}{llllll}
        Env. & Dist. & $\sigma$ & c &\multicolumn{2}{c}{Rewards}\\
        &&&& Proto & One Hot \\
        \hline
        \multirow{2}{*}{Tri.} & Unif. & 0.9 & 9 &  -\textbf{124} (12) & -201 (10)   \\
          & Skew. & 0.9 & 9 & -\textbf{100} (4) & -136 (9)  \\
         \cline{2-6}
        \multirow{3}{*}{9-P.}& Unif. & 0.0 & 3 & -\textbf{60} (7) & -96 (10)  \\
        & Unif & 0.0 &9 &  -65 (2)  & -66 (6) \\
        & Unif & 0.5 &9 &  -\textbf{100} (10) & -132 (2) \\
        \hline
    \end{tabular}
	\end{minipage}\hfill
	\begin{minipage}{0.45\linewidth}
		\centering
		\includegraphics[width=\textwidth]{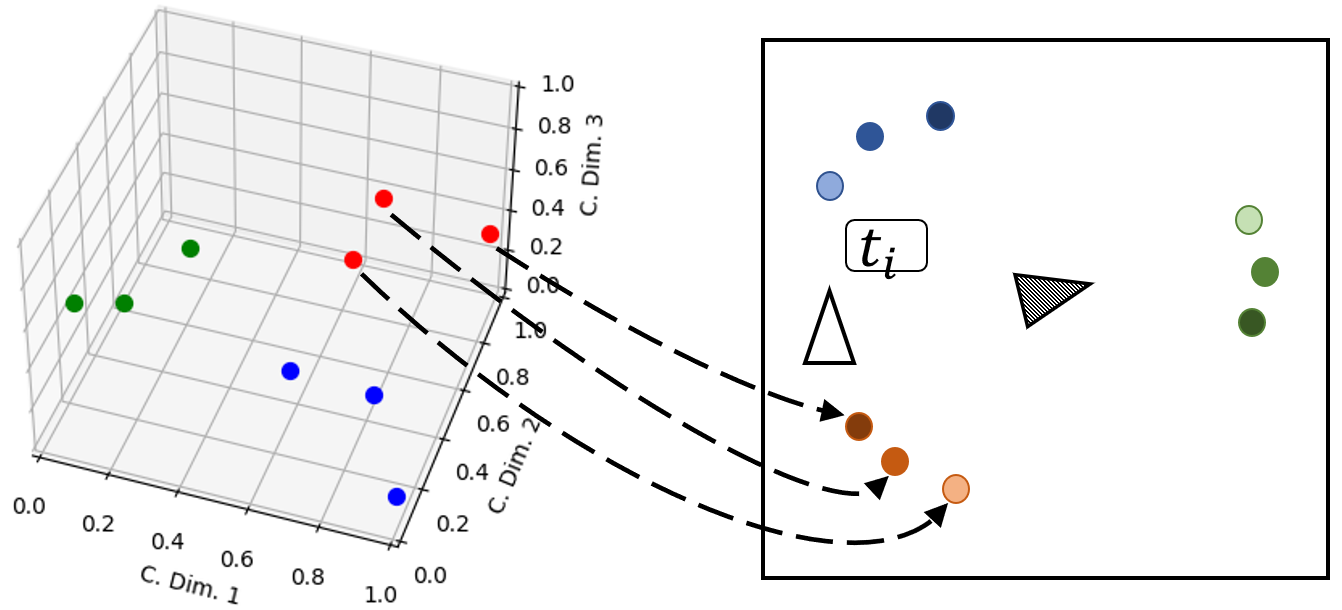}
		\captionof{figure}{The learned prototypes (left) reflect structure in the 9 points environment (right), with similar prototypes encoding similar locations. Prototypes taken from a $c=3; z=9; \sigma=0.3$ experiment conducted for visualization.}
		\label{fig:prototype_locations}
	\end{minipage}
\end{table}

The \textit{triangle} environment results reveal two benefits of our technique over \textit{one-hot}: robustness to noise when $z < c$ and the ability to exploit a non-uniform prior.
First, when there were fewer tokens than communication dimensions, agents used the additional dimensions for robust communication, as shown in the top row of Table~\ref{tab:robustness}.
The three prototypes spread into distant corners of the 9-dimensional unit hypercube, creating a larger distance between tokens than one-hot encodings.
Next, the second row shows that our technique was able to benefit from the non-uniform prior over landmarks by yielding greater robustness for tokens denoting the more likely landmarks.
The one-hot technique, in contrast, learned to always go to the location between the two most likely landmarks.

The \textit{9 points} environment enabled us to study two other attributes of learned communication tokens: the benefits of greater expressivity in a channel-limited domain ($z > c$) and clustering behavior due to a non-identity reward matrix when $z = c$.
First, we confirmed in the third row of Table~\ref{tab:robustness} that, in an environment with $c = 3$ but 9 landmarks, \textit{one-hot} lacked the expressivity to uniquely denote each target; our technique decoupled the number of tokens and the communication dimension, allowing 9 discrete tokens in a 3D communication channel.
The baseline of $z = c = 9$, in the fourth row, confirmed that the 9 prototypes in 3D achieved the same reward as 9 prototypes in 9D or 9 one-hot tokens.
Lastly, in the fifth row, we observed that increasing environment noise led to a graceful degradation of performance for prototype-based agents, but worse failures for \textit{one-hot}.
Inspection of the learned prototypes reveals why: tokens denoting landmarks in the same cluster of 3 landmarks converged into clusters in the communication space, allowing the clusters to separate from each other, as shown in Figure~\ref{fig:prototype_locations}.
Thus, intra-cluster errors became more likely, but inter-cluster error likelihood decreased.
In fact, the mean reward of -100 achieved in the noisy environment is similar to the reward of -96 achieved by the 3D one-hot communication method: both of these techniques devolved to the behavior of treating clusters as single points.

In our final environment, \textit{uniform}, a seeker and observer agent were spawned at random locations in the world; two landmarks (target and distractor) were also spawned at locations drawn from a uniform distribution over the world bounds.
The observer agent observed the location of both landmarks; the seeker agent observed its own position and the id of the target (i.e., the first or second landmark).
Each agent communicated to the other at each timestep.
To achieve high reward, designated as the negative distance from the seeker to the target, both agents had to communicate: the seeker telling the observer which landmark was the target, and the observer responding with the location of the target.
This environment, unlike the prior two, allowed us to assess how agents used a discrete vocabulary to communicate about continuous data.
Prototype and \textit{cont} agents communicated with $c = 3$; \textit{one-hot} used $c = 5$, as prototype agents used $z = 5$.

In the \textit{uniform} environment, we found that our method discretized continuous data in a reasonable way and that the emergent tokens enabled some measure of interpolation between representations.
First, our method and \textit{one-hot} performed similarly to each other and worse than \textit{cont} (Table~\ref{tab:unif_results}).
Given that the observer agent had to communicate about continuous values (the target's location) with a finite, discrete vocabulary, this is unsurprising.
Plotting trajectories of a seeker agent over 200 episodes, the discrete pattern of locations emerges (Figure~\ref{fig:paths}).

\begin{table}
	\begin{minipage}{0.3\linewidth}
		\caption{Median (std. error) reward over 5 runs in \textit{Uniform}. Continuous communication outperformed both discrete methods, which converged to similar discretization behaviors of the continuous space.}
		\label{tab:unif_results}
		\centering
        \begin{tabular}{lr}
            Method & Reward \\
            \hline
            Cont. & \textbf{-97.8} (3.8) \\
            Proto & -113.6 (3.8)\\
            One Hot & -111.8 (3.4)\\
        \end{tabular}
	\end{minipage}
	\hfill
	\begin{minipage}{0.32\linewidth}
		\centering
		\includegraphics[width=\textwidth,trim=25 20 40 20,clip]{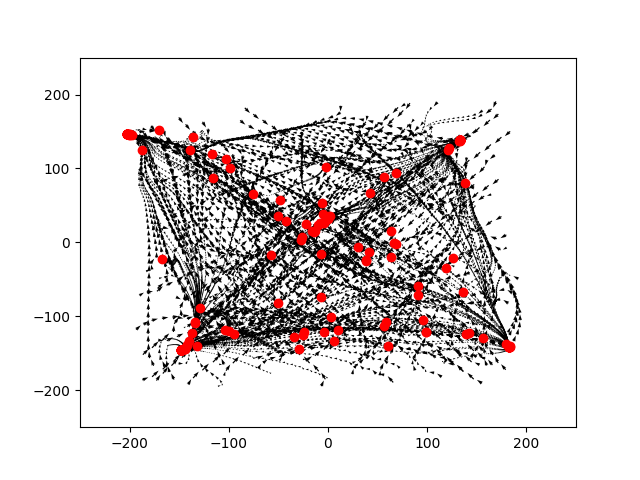}
		\captionof{figure}{Tracing seeker paths (black lines) and final locations (red dots) revealed the discretization imposed by prototypes.}
		\label{fig:paths}
	\end{minipage}
	\hfill
	\begin{minipage}{0.32\linewidth}
		\centering
		\begin{overpic}[width=\textwidth,trim=25 15 35 20,clip]{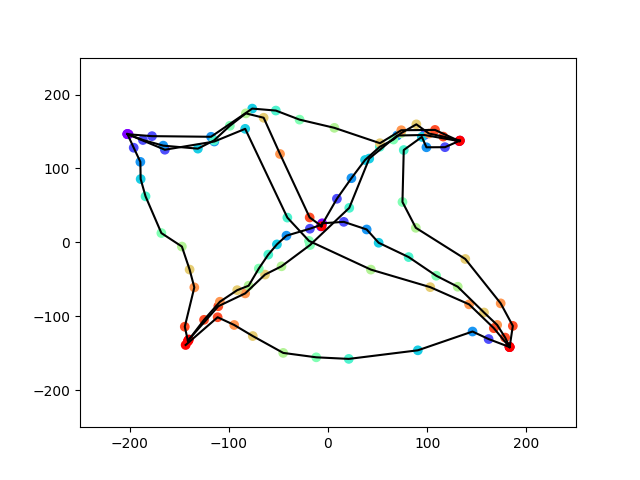}
         \put (10,60) {\large$\displaystyle\alpha$}
         \put (80,55) {\large$\displaystyle\beta$}
         \put (35,65) {\large$\displaystyle0.5\alpha + 0.5\beta$}
         \put (41, 61) {$\bullet$}
        \end{overpic}
		\captionof{figure}{Mean destination of the seeker agent when given pairwise interpolations of prototypes from 0\% (red) to 100\% (purple).}
		\label{fig:interpolation}
	\end{minipage}
\end{table}


Lastly, we tested prototype-based agents with interpolations of learned tokens by overriding the observer agent's communication with pairwise interpolations of all tokens at 10\% intervals.
An example of the resulting behavior is plotted in Figure~\ref{fig:interpolation}.
(We note that interpolation for continuous agents was impossible because of the lack of discrete tokens, and \textit{one-hot} precludes communication other than one-hot vectors.)
Interestingly, we found that prototype-based seeker agents often responded to the interpolation of tokens with interpolation of behavior: for example, when observing the mean of two tokens, $\alpha$ and $\beta$, the seeker agent went to the mean of the locations it would have gone to when observing just $\alpha$ or just $\beta$.
These initial findings motivated work in the next section on zero-shot understanding of natural language.

Together, these three experiments provided support for our first hypothesis, \textbf{H1}.
Our technique enabled agents to learn tokens that provided greater robustness to noise in the \textit{triangle} and \textit{9 points} environments by separating and clustering tokens, confirmed visually in Figure~\ref{fig:prototype_locations}.

\section{Human-Agent Experiments}
\label{sec:human}

Lastly, we tested how well our technique supported aspects of human-agent interaction in zero-shot settings.
Results from the prior section already indicated that prototype-based agents learned a sort of semantic communication space: in \textit{9-points}, for example, similar target locations were encoded via similar prototypes, and in \textit{uniform}, the seeker agent responded to interpolations of prototypes with interpolations of behavior.
In order to enable human-agent interaction in experiments in this section, we aligned parts of human and emergent communication.
We found that this alignment, in conjunction with semantically-meaningful communication spaces, supported two forms of zero-shot understanding.
First, agents were able to understand novel human communication at test time (\textbf{H2}).
Second, humans could understand emergent communication embeddings that agents learned (\textbf{H3}).

We conducted these experiments via a reference game, a standard type of environment for emergent communication \citep{hu2020other,bullard2020exploring}.
In our environment, a speaker observed an image, drawn from the CIFAR10 dataset, and emitted a communication vector \citep{cifar10}.
A listener agent observed the original image and a ``distractor'' image that belonged to a different class, as well as the communication emitted by the speaker.
The listener predicted in the second, final timestep which of the two images was the ``target'' observed by the speaker.
Reward was defined via a matrix, $R_{10 \times 10}$, where $R[i, j]$ set the shared reward if the target image belonged to class $i$ and the listener agent predicted an image of class $j$.
We penalized some errors more than others according to high-level groupings of classes: correct predictions got reward 1, confusion within the set of 6 animal classes or 4 vehicle classes got $0.75$, and predictions that belonged to the wrong group got $0$.
We transformed this basic reference game into two zero-shot experiments.

\subsection{Zero-Shot Agent Understanding}
In our first zero-shot experiment, we assessed agents' ability to understand novel communication by withholding two image classes (one animal and one vehicle) from training.
We trained 5 teams from scratch for each of the 24 possible animal-vehicle excluded combinations, resulting in 120 teams.
All agents were trained in environments with $\sigma = 0.05$; \textit{cont} and prototype agents used $c = 3$, prototype agents used $z = 10$, and \textit{one-hot} used $c = 10$.

To align human and agent communication, we augmented self-play training with 100 supervised examples of image-communication pairs.
These data were interleaved with self-play in training agents, biasing them to match the desired policy, as done in prior social convention literature \citep{lerer2019learning,tucker2020adversarially,lowe2020interaction}.
For \textit{cont} and prototype agents, we used a BERT-like model to create communication by embedding the class label names \citep{reimers-2019-sentence-bert}; \textit{one-hot} used the classification label.
An additional set of prototype agents was trained with hand-designed word embeddings that clustered vehicle and animal embeddings.
See Appendix~\ref{app:nlp} for details on embeddings and image feature extraction.
For the prototype agents, we used the lexicon-setting approach from Section~\ref{sec:shaping} to set $T$ equal to the label embeddings. 

During evaluation, we measured listener prediction accuracy over 500 trials when using different types of communication or images.
In ``in-distribution'' experiments, we evaluated models using the 8 image classes seen during training; in ``out-of-distribution'' experiments, we used only the 2 image classes never seen in training.
Furthermore, we tested three types of communication vectors: those produced by the trained speaker agent (Self), by embedding the class label (Label), or by using embeddings of captions generated by human annotators on Amazon Mechanical Turk (AMT).
Details of the user study for generating AMT descriptions are included in Appendix~\ref{app:amt}; in essence, participants were asked to provide descriptions of images, replicating the role of the speaker agent.

\begin{table}
  \caption{CIFAR10 zero-shot reference game mean reward (std. error). Prototype-based models were the only discrete communication method that enabled better-than-random zero-shot understanding.}
  \label{tab:reference_results}
  \centering
  \begin{tabular}{lllllll}
    \toprule
    Method     &  Label In & Label Out & Self In & Self Out & AMT In & AMT Out \\
    \midrule
    Cont (BERT) & 91\% (0.7) & 55\% (1.0)& 91\% (0.7) & 68\% (1.0)& 73\% (0.4) & 53\% (0.6) \\
    Proto (BERT) & 93\% (0.7) & 58\% (1.0) & 93\% (0.7) &66\% (0.9) & 70\% (0.4) & 56\% (0.6) \\
    Proto (Hand) & 75\% (1.1) & 62\% (1.2) & 80\% (0.6) & 58\% (1.0) & & \\
    One-hot &  96\% (0.2) & 50\% (0.8) & 96\% (0.2)& 69\% (1.0) & &\\
    \bottomrule
  \end{tabular}
\end{table}

Results, shown in Table~\ref{tab:reference_results},  supported hypothesis \textbf{H2}: prototype listener agents supported two types of zero-shot communication.
AMT In experiments reflected one type of zero-shot understanding by forcing listeners to interpret novel communication at test time.
Despite training with a small, finite vocabulary, prototype listeners performed almost on par with \textit{cont} agents.
Conversely, one-hot encoding precludes embedding novel communication.
AMT Out and Label Out results reflected a second type of zero-shot understanding: listeners simultaneously observed novel images and novel communication.
In these settings,  our method outperformed \textit{cont} and the random performance of \textit{one-hot}.
By testing prototype agents trained using BERT or hand-generated embeddings, we found that different embedding methods played an important role in agent performance.
Lastly, in ablation studies, we confirmed that both the reward matrix and environment noise contributed to agents' zero-shot understanding (Appendix~\ref{app:reference}).

\subsection{Zero-Shot Human Understanding of Tokens}
In our final experiment, we studied whether humans could understand agent communication in zero-shot settings, is essence treating humans as listeners in the reference game.
We trained prototype agents in the standard reference game, communicating in 2 dimensions with all 10 classes and no explicit guidance on the tokens.
See Appendix~\ref{app:hyper} for full training details.
In evaluation, we showed participants 8 of the 10 learned tokens, annotated with English class labels.
Finally, we plotted one of the held-out tokens with no label and showed participants two images, one for each of the two held out classes, and asked participants to select the image they thought the new communication most likely referred to.
We recorded participants' selection accuracies among animal-vehicle image pairs for 5 prototype-based models trained from scratch, and a 2D PCA visualization of BERT-based embeddings of labels.

Averaged across the 5 prototype models, with an average of 300 responses per model, participants selected the correct image 70\% of the time (standard error of 5\%).
This compares favorably to the 66\% achieved using embeddings derived from the BERT-based model ($\chi^2(1,N=2093) = 3.21, p = .073$).
While prior art has studied how humans interpret labeled agent tokens in new contexts (e.g., interpreting a one-hot token for ``dolphin'' as referring to a photo of water in Lazaridou et al. \cite{lazaridou2016multi}), we believe this is the first demonstration of human participants understanding novel tokens.
Ultimately, these results provided support for our final hypothesis, \textbf{H3}.


\section{Contributions}
\label{sec:contributions}
In this work, we proposed a technique to address a current gap in existing discrete emergent communication literature.
Prior work in this domain relies upon one-hot tokens for communication, but insights from NLP, as well as decision-theoretic analysis, demonstrate why such an encoding scheme is inadequate in many settings.
Using a technique we developed to enable learnable, discrete tokens, we demonstrated benefits in self-play, including increased robustness to environment noise.
In human-subject experiments, we showed improvements over the random performance of prior art in zero-shot understanding of human communication.
Lastly, we demonstrated that prototype agents learned emergent communication tokens that were human-interepretable and semantically meaningful.
Overall our experiments provided support that our prototype-based communication method produced semantically meaningful communication spaces in noisy environments with non-identity reward matrices, which in turn enabled zero-shot understanding of communication.

Our work provides a first step towards connecting theories of emergent communication and word embeddings.
While we have taken a step in this direction through our proposed technique, future work leveraging insights from discrete representation learning could likely further improve upon our results.
In addition, many of the techniques used to examine word embeddings could be used to analyze emergent communication tokens.
Specifically, we note that prior literature has found reinforced biases in word embeddings; similar care of encoding desirable properties in emergent tokens must be taken.
Lastly, further experiments examining environmental factors that contribute to zero-shot understanding of discrete communication could yield important insight into effective agent training and language evolution.

\begin{ack}
We thank Jacob Andreas for early discussions about embeddings in NLP, and the NeurIPS reviewers for their careful feedback.
This work has been funded by ARL award W911NF-19-2-0146.


\end{ack}


\medskip

\bibliography{sample}

\begin{thebibliography}{10}

\bibitem{almeida2019word}
Felipe Almeida and Geraldo Xex{\'e}o.
\newblock Word embeddings: A survey.
\newblock {\em arXiv preprint arXiv:1901.09069}, 2019.

\bibitem{bengio2003neural}
Yoshua Bengio, R{\'e}jean Ducharme, Pascal Vincent, and Christian Janvin.
\newblock A neural probabilistic language model.
\newblock {\em The journal of machine learning research}, 3:1137--1155, 2003.

\bibitem{bernstein2002complexity}
Daniel~S Bernstein, Robert Givan, Neil Immerman, and Shlomo Zilberstein.
\newblock The complexity of decentralized control of markov decision processes.
\newblock {\em Mathematics of operations research}, 27(4):819--840, 2002.

\bibitem{bullard2021quasi}
Kalesha Bullard, Douwe Kiela, Joelle Pineau, and Jakob Foerster.
\newblock Quasi-equivalence discovery for zero-shot emergent communication.
\newblock {\em arXiv preprint arXiv:2103.08067}, 2021.

\bibitem{bullard2020exploring}
Kalesha Bullard, Franziska Meier, Douwe Kiela, Joelle Pineau, and Jakob
  Foerster.
\newblock Exploring zero-shot emergent communication in embodied multi-agent
  populations.
\newblock {\em arXiv preprint arXiv:2010.15896}, 2020.

\bibitem{chaabouni2019antiefficient}
Rahma Chaabouni, Eugene Kharitonov, Emmanuel Dupoux, and Marco Baroni.
\newblock Anti-efficient encoding in emergent communication.
\newblock In H.~Wallach, H.~Larochelle, A.~Beygelzimer, F.~d\textquotesingle
  Alch\'{e}-Buc, E.~Fox, and R.~Garnett, editors, {\em Advances in Neural
  Information Processing Systems}, volume~32. Curran Associates, Inc., 2019.

\bibitem{chen2018zero}
Long Chen, Hanwang Zhang, Jun Xiao, Wei Liu, and Shih-Fu Chang.
\newblock Zero-shot visual recognition using semantics-preserving adversarial
  embedding networks.
\newblock In {\em Proceedings of the IEEE Conference on Computer Vision and
  Pattern Recognition}, pages 1043--1052, 2018.

\bibitem{NEURIPS2019_fe5e7cb6}
Tom Eccles, Yoram Bachrach, Guy Lever, Angeliki Lazaridou, and Thore Graepel.
\newblock Biases for emergent communication in multi-agent reinforcement
  learning.
\newblock In H.~Wallach, H.~Larochelle, A.~Beygelzimer, F.~d\textquotesingle
  Alch\'{e}-Buc, E.~Fox, and R.~Garnett, editors, {\em Advances in Neural
  Information Processing Systems}, volume~32. Curran Associates, Inc., 2019.

\bibitem{foerster2016learning}
Jakob Foerster, Ioannis~Alexandros Assael, Nando de~Freitas, and Shimon
  Whiteson.
\newblock Learning to communicate with deep multi-agent reinforcement learning.
\newblock In D.~Lee, M.~Sugiyama, U.~Luxburg, I.~Guyon, and R.~Garnett,
  editors, {\em Advances in Neural Information Processing Systems}, volume~29.
  Curran Associates, Inc., 2016.

\bibitem{hausknecht2016cooperation}
Matthew~John Hausknecht.
\newblock {\em Cooperation and communication in multiagent deep reinforcement
  learning}.
\newblock PhD thesis, 2016.

\bibitem{havrylov2017emergence}
Serhii Havrylov and Ivan Titov.
\newblock Emergence of language with multi-agent games: Learning to communicate
  with sequences of symbols.
\newblock In I.~Guyon, U.~V. Luxburg, S.~Bengio, H.~Wallach, R.~Fergus,
  S.~Vishwanathan, and R.~Garnett, editors, {\em Advances in Neural Information
  Processing Systems}, volume~30. Curran Associates, Inc., 2017.

\bibitem{hu2020other}
Hengyuan Hu, Adam Lerer, Alex Peysakhovich, and Jakob Foerster.
\newblock ``{O}ther-play” for zero-shot coordination.
\newblock In {\em International Conference on Machine Learning}, pages
  4399--4410. PMLR, 2020.

\bibitem{shariqbal}
Shariq Iqbal.
\newblock Maddpg-pytorch.
\newblock \url{https://github.com/shariqiqbal2810/maddpg-pytorch}, 2018.

\bibitem{jang2017categorical}
Eric Jang, Shixiang Gu, and Ben Poole.
\newblock Categorical reparameterization with gumbel-softmax, 2017.

\bibitem{kottur2017natural}
Satwik Kottur, Jos{\'e} Moura, Stefan Lee, and Dhruv Batra.
\newblock Natural language does not emerge {`}naturally{'} in multi-agent
  dialog.
\newblock In {\em Proceedings of the 2017 Conference on Empirical Methods in
  Natural Language Processing}, pages 2962--2967, Copenhagen, Denmark,
  September 2017. Association for Computational Linguistics.

\bibitem{cifar10}
Alex Krizhevsky, Vinod Nair, and Geoffrey Hinton.
\newblock Cifar-10 (canadian institute for advanced research).
\newblock 2010.

\bibitem{kuhn1951proceedings}
H.~W. Kuhn and A.~W. Tucker.
\newblock Nonlinear programming.
\newblock In {\em Proceedings of the {S}econd {B}erkeley {S}ymposium on
  {M}athematical {S}tatistics and {P}robability, 1950}, pages 481--492,
  Berkeley and Los Angeles, 1951. University of California Press.

\bibitem{lazaridou2020emergent}
Angeliki Lazaridou and Marco Baroni.
\newblock Emergent multi-agent communication in the deep learning era.
\newblock {\em arXiv preprint arXiv:2006.02419}, 2020.

\bibitem{lazaridou2018emergence}
Angeliki Lazaridou, Karl~Moritz Hermann, Karl Tuyls, and Stephen Clark.
\newblock Emergence of linguistic communication from referential games with
  symbolic and pixel input.
\newblock In {\em 6th International Conference on Learning Representations,
  {ICLR} 2018, Vancouver, BC, Canada, April 30 - May 3, 2018, Conference Track
  Proceedings}. OpenReview.net, 2018.

\bibitem{lazaridou2016multi}
Angeliki Lazaridou, Alexander Peysakhovich, and Marco Baroni.
\newblock Multi-agent cooperation and the emergence of (natural) language.
\newblock In {\em 5th International Conference on Learning Representations,
  {ICLR} 2017, Toulon, France, April 24-26, 2017, Conference Track
  Proceedings}. OpenReview.net, 2017.

\bibitem{lazaridou-etal-2020-multi}
Angeliki Lazaridou, Anna Potapenko, and Olivier Tieleman.
\newblock Multi-agent communication meets natural language: Synergies between
  functional and structural language learning.
\newblock In {\em Proceedings of the 58th Annual Meeting of the Association for
  Computational Linguistics}, pages 7663--7674, Online, July 2020. Association
  for Computational Linguistics.

\bibitem{lee2019countering}
Jason Lee, Kyunghyun Cho, and Douwe Kiela.
\newblock Countering language drift via visual grounding.
\newblock In {\em EMNLP-IJCNLP 2019 - 2019 Conference on Empirical Methods in
  Natural Language Processing and 9th International Joint Conference on Natural
  Language Processing, Proceedings of the Conference}, EMNLP-IJCNLP 2019 - 2019
  Conference on Empirical Methods in Natural Language Processing and 9th
  International Joint Conference on Natural Language Processing, Proceedings of
  the Conference, pages 4385--4395. Association for Computational Linguistics,
  2020.

\bibitem{lee2017emergent}
Jason Lee, Kyunghyun Cho, Jason Weston, and Douwe Kiela.
\newblock Emergent translation in multi-agent communication.
\newblock January 2018.
\newblock 6th International Conference on Learning Representations, ICLR 2018 ;
  Conference date: 30-04-2018 Through 03-05-2018.

\bibitem{lerer2019learning}
Adam Lerer and Alexander Peysakhovich.
\newblock Learning existing social conventions via observationally augmented
  self-play. aaai.
\newblock In {\em ACM conference on Artificial Intelligence, Ethics, and
  Society}, 2019.

\bibitem{lewis2008convention}
David Lewis.
\newblock {\em Convention: A philosophical study}.
\newblock John Wiley \& Sons, 2008.

\bibitem{lowe2020interaction}
Ryan Lowe, Abhinav Gupta, Jakob~N. Foerster, Douwe Kiela, and Joelle Pineau.
\newblock On the interaction between supervision and self-play in emergent
  communication.
\newblock In {\em ICLR}. OpenReview.net, 2020.

\bibitem{lowe2017multi}
Ryan Lowe, Yi~Wu, Aviv Tamar, Jean Harb, Pieter Abbeel, and Igor Mordatch.
\newblock Multi-agent actor-critic for mixed cooperative-competitive
  environments.
\newblock {\em Neural Information Processing Systems (NIPS)}, 2017.

\bibitem{meng2020unsupervised}
Qingjie Meng, Daniel Rueckert, and Bernhard Kainz.
\newblock Unsupervised cross-domain image classification by distance metric
  guided feature alignment.
\newblock In {\em Medical Ultrasound, and Preterm, Perinatal and Paediatric
  Image Analysis}, pages 146--157. Springer, 2020.

\bibitem{word2vec}
Tomas Mikolov, Ilya Sutskever, Kai Chen, Greg~S Corrado, and Jeff Dean.
\newblock Distributed representations of words and phrases and their
  compositionality.
\newblock In C.~J.~C. Burges, L.~Bottou, M.~Welling, Z.~Ghahramani, and K.~Q.
  Weinberger, editors, {\em Advances in Neural Information Processing Systems},
  volume~26. Curran Associates, Inc., 2013.

\bibitem{mnih2016asynchronous}
Volodymyr Mnih, Adria~Puigdomenech Badia, Mehdi Mirza, Alex Graves, Timothy
  Lillicrap, Tim Harley, David Silver, and Koray Kavukcuoglu.
\newblock Asynchronous methods for deep reinforcement learning.
\newblock In {\em International conference on machine learning}, pages
  1928--1937. PMLR, 2016.

\bibitem{mordatch2017emergence}
Igor Mordatch and Pieter Abbeel.
\newblock Emergence of grounded compositional language in multi-agent
  populations.
\newblock {\em Proceedings of the AAAI Conference on Artificial Intelligence},
  32(1), Apr. 2018.

\bibitem{glove}
Jeffrey Pennington, Richard Socher, and Christopher~D Manning.
\newblock Glove: Global vectors for word representation.
\newblock In {\em Proceedings of the 2014 conference on empirical methods in
  natural language processing (EMNLP)}, pages 1532--1543, 2014.

\bibitem{reimers-2019-sentence-bert}
Nils Reimers and Iryna Gurevych.
\newblock Sentence-bert: Sentence embeddings using siamese bert-networks.
\newblock In {\em Proceedings of the 2019 Conference on Empirical Methods in
  Natural Language Processing}. Association for Computational Linguistics, 11
  2019.

\bibitem{schulman2017proximal}
John Schulman, Filip Wolski, Prafulla Dhariwal, Alec Radford, and Oleg Klimov.
\newblock Proximal policy optimization algorithms.
\newblock {\em CoRR}, abs/1707.06347, 2017.

\bibitem{socher2013zero}
Richard Socher, Milind Ganjoo, Christopher~D. Manning, and Andrew~Y. Ng.
\newblock Zero-shot learning through cross-modal transfer.
\newblock In {\em Proceedings of the 26th International Conference on Neural
  Information Processing Systems - Volume 1}, NIPS'13, page 935–943, Red
  Hook, NY, USA, 2013. Curran Associates Inc.

\bibitem{sukhbaatar2016learning}
Sainbayar Sukhbaatar, arthur szlam, and Rob Fergus.
\newblock Learning multiagent communication with backpropagation.
\newblock In D.~Lee, M.~Sugiyama, U.~Luxburg, I.~Guyon, and R.~Garnett,
  editors, {\em Advances in Neural Information Processing Systems}, volume~29.
  Curran Associates, Inc., 2016.

\bibitem{tucker2020adversarially}
Mycal Tucker, Yilun Zhou, and Julie Shah.
\newblock Adversarially guided self-play for adopting social conventions.
\newblock {\em arXiv preprint arXiv:2001.05994}, 2020.

\bibitem{xian2018zero}
Yongqin Xian, Christoph~H Lampert, Bernt Schiele, and Zeynep Akata.
\newblock Zero-shot learning—a comprehensive evaluation of the good, the bad
  and the ugly.
\newblock {\em IEEE transactions on pattern analysis and machine intelligence},
  41(9):2251--2265, 2018.

\bibitem{yu2018deep}
Fisher Yu, Dequan Wang, Evan Shelhamer, and Trevor Darrell.
\newblock Deep layer aggregation.
\newblock In {\em Proceedings of the IEEE conference on computer vision and
  pattern recognition}, pages 2403--2412, 2018.

\end{thebibliography}

\section*{Checklist}


\begin{enumerate}

\item For all authors...
\begin{enumerate}
  \item Do the main claims made in the abstract and introduction accurately reflect the paper's contributions and scope?
    \answerYes{}{}
  \item Did you describe the limitations of your work?
    \answerYes{See Section~\ref{sec:method} on learning, and Section~\ref{sec:contributions} on broader limitations.}
  \item Did you discuss any potential negative societal impacts of your work?
    \answerYes{See Section~\ref{sec:contributions} on how word embeddings and our technique could reinforce biases.}
  \item Have you read the ethics review guidelines and ensured that your paper conforms to them?
    \answerYes{}
\end{enumerate}

\item If you are including theoretical results...
\begin{enumerate}
  \item Did you state the full set of assumptions of all theoretical results?
    \answerYes{In our decision-theoretic analysis, we analyze what assumptions lead to one-hot tokens being optimal communication tokens.}
	\item Did you include complete proofs of all theoretical results?
    \answerYes{See Appendix~\ref{app:proof}}
\end{enumerate}

\item If you ran experiments...
\begin{enumerate}
  \item Did you include the code, data, and instructions needed to reproduce the main experimental results (either in the supplemental material or as a URL)?
    \answerYes{See footnote 1}
  \item Did you specify all the training details (e.g., data splits, hyperparameters, how they were chosen)?
    \answerYes{See Appendices.}
	\item Did you report error bars (e.g., with respect to the random seed after running experiments multiple times)?
    \answerYes{See standard error, reported in all tables.}
	\item Did you include the total amount of compute and the type of resources used (e.g., type of GPUs, internal cluster, or cloud provider)?
    \answerYes{See Appendix~\ref{app:hyper}}
\end{enumerate}

\item If you are using existing assets (e.g., code, data, models) or curating/releasing new assets...
\begin{enumerate}
  \item If your work uses existing assets, did you cite the creators?
    \answerYes{See Section~\ref{sec:baselines}}
  \item Did you mention the license of the assets?
    \answerYes{See Section~\ref{sec:baselines}}
  \item Did you include any new assets either in the supplemental material or as a URL?
    \answerNA{No new assets}
  \item Did you discuss whether and how consent was obtained from people whose data you're using/curating?
    \answerYes{Data were not published; details on user studies are in Appendices.}
  \item Did you discuss whether the data you are using/curating contains personally identifiable information or offensive content?
    \answerYes{Data were not published and were not identifying, as shown in user study details in Appendices.}
\end{enumerate}

\item If you used crowdsourcing or conducted research with human subjects...
\begin{enumerate}
  \item Did you include the full text of instructions given to participants and screenshots, if applicable?
    \answerYes{See Appendix~\ref{app:amt}.}
  \item Did you describe any potential participant risks, with links to Institutional Review Board (IRB) approvals, if applicable?
    \answerYes{See Appendix~\ref{app:amt}.}
  \item Did you include the estimated hourly wage paid to participants and the total amount spent on participant compensation?
    \answerYes{See Appendix~\ref{app:amt}.}
\end{enumerate}

\end{enumerate}


\appendix

\section{Appendix: Optimality of One-Hot Encodings}
\label{app:proof}

We present a brief proof about the local optimality of one-hot encodings in the decision-theoretic framework presented in Section~\ref{sec:analysis}.
We seek to prove that, under assumptions of an identity reward matrix, tokens constrained to a unit hypercube, and gaussian additive noise, one-hot tokens are an optimally robust communication strategy.
We only seek to prove local optimality, as one many trivially generate multiple, equally optimal tokens by, for example, flipping all bits.

The following derivation uses Karush–Kuhn–Tucker (KKT) conditions, a generalization of Lagrange multipliers \citep{kuhn1951proceedings}.

\begin{theorem}[Optimality of One-Hot Encodings]
A matrix $T_{c\times c}$, constrained to the unit hypercube, locally maximizes $\sum_{i \in [1, c]} \sum_{j \in [1, c]; j \neq i} (T[i] - T[j])^\top(T[i] - T[j])$ using one-hot encodings.
\end{theorem}

\begin{proof}
We maximize the function, subject to constraints.
We denote the $i^{th}$ row of matrix $T$ as $T_i$, and the $j^{th}$ element of the $i^{th}$ row as $T_{ij}$.

\begin{alignat}{3}
    T^* & = \argmax_T \sum_{i \in [1, c]} \sum_{j \in [1, c]; j \neq i} (T_i - T_j)^\top(T_i - T_j)  \qquad s.t. \quad 0 \leq T_{ij} \leq 1 \forall i, j \in [1, c]\\
    &\mbox{We define KKT terms, using } \mu_i \mbox{ and } \lambda_i \mbox{ for constraints. }\mu_{ij}, \lambda_{ij} \geq 0 \forall i, j \in [1, c] \\
    \mathcal{L} &= \sum_{i \in [1, c]} \sum_{j \in [1, c]; j \neq i} \left[ (T_i - T_j)^\top(T_i - T_j)\right] - \sum_{i \in [1, c]} \left[\vec{\mu_i}(T_i - \vec{1}) - \vec{\lambda_i} T_i \right]&& \\
    \frac{\partial \mathcal{L}}{\partial \vec{\mu_i}} &= -T_i + \vec{1} = \vec{0}\\
    &T_{ij} = 1 \quad \mbox{if} \quad \mu_{ij} \neq 0\\
    \frac{\partial \mathcal{L}}{\partial \vec{\lambda_i}} &= T_i = \vec{0}\\
    &T_{ij} = 0 \quad \mbox{if} \quad \lambda_{ij} \neq 0\\
    \frac{\partial \mathcal{L}}{\partial T_i} &= \sum_{j \in [1, c]; j \neq i} \left[ \frac{\partial ||T_i||^2}{\partial T_i} -  \frac{\partial T_i^\top T_j}{\partial T_i} - \frac{\partial T_j^\top T_i}{\partial T_i} + \frac{\partial ||T_j||^2}{\partial T_i}\right] - \vec{\mu_i} + \vec{\lambda_i} = \vec{0}\\
\end{alignat}

We seek to show that one-hot vectors are an optimum, so we now show that one-hot vectors indeed respect the constraints and set the derivatives to zero. By assuming one-hot vectors, we know that $T_i^\top T_j = 0 \forall i \neq j.$
\begin{alignat}{3}
    \frac{\partial \mathcal{L}}{\partial T_i} = \sum_{j \in [1, c]; j \neq i} \left[ \frac{\partial ||T_i||^2}{\partial T_i} \right]  - \vec{\mu_i} + \vec{\lambda_i} &= \vec{0}\\
    \sum_{j \in [1, c]; j \neq i} \left[2T_i\right] - \vec{\mu_i} + \vec{\lambda_i} &= \vec{0}\\
    2(c - 1)T_i - \vec{\mu_i} + \vec{\lambda_i} &= \vec{0}\\
\end{alignat}

We now consider each element of $T_{ij}$, again assuming that $T_i$ is a one-hot vector.

\begin{alignat}{3}
    \mbox{If}\quad T_{ij} &= 0\\
    &\rightarrow 2(c - 1)T_{ij} = 0& \mbox{Because } T_{ij} = 0\\
    &\rightarrow \mu_{ij} = 0 & \mbox{Because the constraint is not active}\\
    &\rightarrow \lambda_{ij} = 0 & \mbox{Can set } \lambda_{ij} \mbox{ to satisfy equation} \\
    \mbox{If}\quad T_{ij} &= 1\\
    &\rightarrow \lambda_{ij} = 0 & \mbox{Because the constraint is not active}\\
    &\rightarrow 2(c-1)T_{ij} - \mu_{ij} = 0\qquad &\mbox{Solve with } \mu_{ij} = 2(c - 1) \mbox{ to satisfy equation}
\end{alignat}

Thus, we have shown that one-hot vectors in the unit hypercube are a solution to the constrained optimization problem, indicating that it is a local optimum.

\end{proof}

One may also prove that one-hot encodings may not be optimal when the assumptions of an identity reward matrix or a uniform prior over tokens are relaxed.
These proofs follow from counter-examples: we seek merely to show that there exist other tokens that achieve a higher expected reward.
We provide two such examples below for the $z = c = 3$ case, calculating the expected reward for one-hot tokens and alternative tokens we propose, which we show achieve higher mean reward.

\subsection{Uniform prior; non identity reward}
\begin{alignat}{2}
    R &= \begin{pmatrix}
1 & 1 & 0 \\
1 & 1 & 0 \\
0 & 0 & 1 \end{pmatrix}\\
    T_{1hot} &= \begin{pmatrix}
1 & 0 & 0 \\
0 & 1 & 0 \\
0 & 0 & 1 \end{pmatrix}\\
    T_{alt} &= \begin{pmatrix}
1 & 1 & 0 \\
1 & 1 & 0 \\
0 & 0 & 1 \end{pmatrix}
\end{alignat}

In this case, there is a reward of 1 even if tokens 0 and 1 are confused.
Making those two tokens identical (in $T_{alt}$) therefore incurs no penalty and instead decreases the likelihood of confusion with token 2, thus increasing expected reward compared to using $T_{1hot}$.
Running 100,000 simulations of this scenario, with $\sigma=0.5$, shows that one-hot tokens achieve a mean reward of 0.91, whereas $T_{alt}$ achieves a reward of 0.96. 

\subsection{Non-uniform prior; identity reward}
\begin{alignat}{2}
    P &= \begin{pmatrix}
    0.499 & 0.499 & 0.002 \end{pmatrix}\\
    T_{1hot} &= \begin{pmatrix}
1 & 0 & 0 \\
0 & 1 & 0 \\
0 & 0 & 1 \end{pmatrix}\\
    T_{alt} &= \begin{pmatrix}
0 & 0 & 0 \\
1 & 1 & 1 \\
0.5 & 0.5 & 0.5 \end{pmatrix}\\
\end{alignat}

Although $R = I$, because tokens 0 and 1 are far more likely than token 2, setting tokens $T$ to maximally distinguish between the first two tokens becomes more important than forcing token 2 to be far from other tokens.
Simulations with $\sigma = 0.5$ show that one-hot achieves reward of 0.92 but $T_{alt}$ achieves 0.96.

Together, these two examples prove that when $R \neq I$ or there is a non-uniform prior over tokens, one-hot encodings may not be optimal.
\section{Appendix: Particle Environment Details}
\label{app:envs}
The \textit{triangle} environment spawned the speaker and listener agents at the origin.
Landmarks were distributed at distance 200 from the origin, every $\frac{\pi}{3}$ radians, starting at 0.
The listener particle had a mass of 0.1; its force actions were bounded via a \texttt{tanh} operation; force was integrated to change velocity and position via a double-integrator physics engine with a damping coefficient of 0.9.
In order to preserve the Markov property of planning, the listener agent observed its own 2D position and 2D velocity.
Episodes lasted 50 timesteps.

The \textit{9-points} environment used the same physics engine as \textit{triangle} and episodes also lasted 50 timesteps.
Landmarks were located at distance 200 from the origin, with landmark $i$ at angle $\frac{\pi}{3} + 0.3\lfloor \frac{i}{3} \rfloor$ radians.
We found that 0.3 was large enough to encourage unique tokens for each landmark, but small enough that clusters were still apparent.

The \textit{uniform} environment used the same physics engine, with both agents' masses as 0.1, but episodes lasted 100 timesteps.
We found that additional time enabled more convergent behaviors: that is, agents had to learn to reach a location and stop for maximal reward, as opposed to traveling at full speed in a fixed direction.
Landmarks and agents were spawned at locations drawn uniformly at random from the square described by $(-200, -200)$ to $(200, 200)$.

\section{Appendix: Training Details and Hyperparameters}
\label{app:hyper}
This section specifies the training hyperparameters used in all experiments.
When conducting multiple trials for a given environment, each trial was conducted by setting the random seed equal to the zero-indexed trial number.

All agents were trained using an implementation of MADDPG \citep{lowe2017multi}, using an existing MIT-licensed PyTorch implementation of the algorithm as an initial development point \citep{shariqbal}.
When experimenting with environmental noise, all agents were trained using identical noise models and evaluated with zero noise.
We used an Adam optimizer with default parameters except for the learning rate, which we set to 0.01.
Unless otherwise noted, we set $\tau = 0.01$ -- the degree of soft updating of the target networks used in MADDPG.

All agents used a three-layer MLP with ReLU activations and hidden dimension 64 as the base of their policy, with communication or action heads transforming the base output to the desirable form via two layers.
Continuous outputs were bounded by a \texttt{tanh} activation.
Tokens used by prototype agents were bounded to the unit hypercube for fair comparisons with one-hot tokens by passing tokens through a sigmoid activation.
In agent-only experiments, we used temperature 0.1; in human-agent experiments, we used temperature 1.0.

All training times were calculated on a desktop computer with an NVIDIA GeForce RTX-2080 graphics card and 16 Intel-i9 CPUs.
We reported training times for a single team - total training time for a suite of experiments is calculated by multiplying the numbers of methods, random seeds, and environment settings.
In total, all experiments and evaluation took approximately 48 hours.

\subsection{Agent-Only}
\subsubsection{Triangle}
Episodes lasted 50 timesteps; training was conducted over 10,000 episodes.
Models were updated every 2 episodes, using a batch size of 1,024, sampled from the replay buffer.

Training took approximately 5 minutes.

\subsubsection{9 Points}
Episodes lasted 50 timesteps; training was conducted over 5,000 episodes in all scenarios.
Models were updated every 2 episodes, using a batch size of 1,024, sampled from the replay buffer.

In the case of $c = 3$, wherein the one-hot agents could only use 3 tokens to communicate about 9 possible targets, we observed some unstable training behavior: agents failed to consistently reach and then stay at the high mean rewards sometimes observed during training.
Experiments with different learning rates and values of $\tau$ over three orders of magnitude failed to resolve this issue.
Because this undesirable behavior only arose in this specific scenario of under-parametrized communication, we were not concerned by the bad performance.
In order to calculate fair baseline values for one-hot communication, we used the best-performing model from each training run, with models saved every 1000 episodes.

Training took approximately 5 minutes.

\subsubsection{Uniform}
Episodes lasted 100 timesteps; training was conducted over 5,000 episodes.
Models were updated every episode; using a batch size of 1,024 sampled from the replay buffer of length $100,000$.

In experiments using continuous or prototype-based communication, we used $c = 3$ and $\sigma = 0$; for one-hot communication, we used $c = 5$ to allow for 5 tokens.
Initial experiments with other values of $\sigma$ at $0.01$ or $0.05$ showed no large differences in behavior.

Training took approximately 10 minutes.

The interpolation results in Figure~\ref{fig:interpolation} were generated by calculating the mean end location of the seeker agent at 10\% interpolation intervals for each pair of prototypes.
Mean locations were calculated over 20 trials for each interpolation level and pair of tokens.

\subsection{Human-Agent Experiments}
In all human-agent experiments, agents were trained in the 2-timestep reference game.
In training, models were updated every 50 episodes, using batch size 1,024 from the replay buffer of length 10,000.
We used a learning rate of 0.01, and a $\tau = 0.0001$.

Rather than feed agents the direct pixels from images, we pre-trained a CIFAR10 image classifier, based on the SimpleDLA architecture \citep{yu2018deep}.
Trained over 200 epochs with an SGD optimizer with learning rate 0.1, momentum 0.9, and weight decay of $5e-4$, using random cropping and horizontal flipping of images in training, it achieved 95\% accuracy.
We used the penultimate layer of the network as a feature extractor; the many-dimensional features were reduced to 10D via principle component analysis (PCA) conducted on the entire training set.
Although applying the PCA transform potentially discarded important information for classification, in practice we found that it was sufficient to enable high task performance.

\subsubsection{Zero-Shot Agent Understanding}
Training was conducted over 20,000 episodes.
The speaker agent used fixed tokens, as explained in Section~\ref{sec:shaping}.
The speaker agent was pretrained for 5,000 epochs on the 100 elements of the supervised dataset, with early stopping with patience 50.
A subsequent experiment with 1,000 supervised datapoints (Appendix~\ref{app:reference}) showed no significant change in performance.
For the continuous and prototype-based methods, we used $c = 3$; for one-hot, we used $c = 10$.
Communication was corrupted with zero-mean Gaussian noise with $\sigma = 0.05$.

Agents using all methods appeared to achieve nearly perfect reference accuracy in training with these parameters, indicating that they had achieved a policy optimum that other hyperparameters would not be able to outperform.
Training took approximately 90 seconds.

When using embeddings of AMT-generated communication, we randomly selected any caption from the set of captions generated for images of the true class of the image.
This likely underestimated the rich discriminative power of the AMT descriptions.
However, given that the AMT captions were generated for images in the CIFAR10 training set, and we evaluated agents with images from the test set, there did not exist an exact match of images and captions.

\subsubsection{Zero-Shot Human Understanding of Tokens}
Training was conducted over 50,000 episodes.
The speaker agent did not use fixed tokens or grounding data, as we wished to study if emergent communication tokens were human-interpretable.
The environment matched that of the previous reference game, but with $\sigma = 0.2$.

Without supervised data, learning successful communication protocols in such settings is challenging \citep{NEURIPS2019_fe5e7cb6}.
We overcame this difficulty by creating supervised data for the penultimate, softmax layer in prototype speaker agents.
That is, we added a crossentropy loss term to agent training to encourage prototype agents to produce a class-specific one-hot vector for images.
As this one-hot vector was then multiplied by the token matrix, $T$, which was not subject to supervised data losses, the form of communication remained unguided.

The ``communication-space'' diagrams, an example of which is shown in Figure~\ref{fig:amt_2}, displayed 8 2D communication tokens with class labels. 
Class labels were generated by finding the most common token that the speaker emitted for each image class over 1,000 evaluation runs.
The two excluded classes always consisted of one animal and one vehicle; one of those was then used to create the ``zero-shot'' token that participants had to interpret.
Participants were then shown two images, one of each held-out class, and were marked as correct if they selected the image that generated the token.
We recorded measured mean participant accuracy.
Full details of the user study are included in Appendix~\ref{app:amt}.

Training took approximately 4 minutes.

\section{Appendix: Reference Game Results}
\label{app:reference}
In addition to the main results on zero-shot understanding in the reference game results, we conducted additional experiments to assess the importance of different hyperparameters used in training our agents.

First, in Table~\ref{tab:reference_results1000}, we repeated the zero-shot experiments, this time training agents with 1000, instead of 100, examples of input images and corresponding tokens.
These additional supervised data had a minimal effect on task performance, indicating general similarity of behavior over a large range of supervised dataset sizes.

\begin{table}
  \caption{CIFAR10 zero-shot results when training agents with 1000 supervised communication vectors. Results were similar to using 100 supervised examples.}
  \label{tab:reference_results1000}
  \centering
  \begin{tabular}{lllllll}
    \toprule
    Method     &  Label In & Label Out & Self In & Self Out & AMT In & AMT Out \\
    \midrule
    Cont (BERT) & 92\% (0.7) & 55\% (0.9) & 93\% (0.7) & 67\% (1.0) & 72\% (0.3) & 52\% (0.6)\\
    Proto (BERT) & 94\% (0.5) & 57\% (1.0)  & 92\% (0.8) & 67\% (1.0) & 75\% (0.3) & 57\% (0.5)\\
    Disc & 98\% (0.2) & 50\% (0.4) & 96\% (0.4) & 68\% (0.9)\\
    \bottomrule
  \end{tabular}
\end{table}

Second, we conducted ablation studies to assess the importance of the reward matrix and noise in the reference game environment.
Results from these studies are included in Table~\ref{tab:ablation}.
We found that both components were useful in training agents with the best zero-shot understanding of human-generated labels (as shown in the Label Out and AMT Out columns).
The results align well with our theory of what enables zero-shot understanding: noise is necessary to induce agents to learn a space of communication instead of just tokens, and a non-identity reward matrix is needed to induce semantic relationships between tokens.

\begin{table}
  \caption{Ablation study for prototype-based communication in the CIFAR10 reference game, using 100 supervised examples. Using both noise and reward enabled the best zero-shot understanding of communication generated from natural language.}
  \label{tab:ablation}
  \centering
  \begin{tabular}{llllllll}
    \toprule
    Noise &  $R \neq I$ &  Label In & Label Out & Self In & Self Out & AMT In & AMT Out\\
    \midrule
    Yes & Yes & 93\% (0.7) & \textbf{58\%} (1.0) & 93\% (0.7) &66\% (0.9) & 70\% (0.4) &\textbf{56\% (0.6)}\\
    Yes & No & 96\% (0.4) & 53\% (1.0) & 96\% (0.4)  & 68\% (1.0) & 76\% (0.2) & 54\% (0.7)\\
    No & Yes & 91\% (0.7) & 55\% (1.0) & 91\% (0.7) & 65\% (1.0) & 71\% (0.4) & 54\% (0.5) \\
    No & No & 96\% (0.4) & 53\% (0.8) & 96\% (0.4) & 67\% (1.0) & 74\% (0.3) & 53\% (0.5) \\
    \bottomrule
  \end{tabular}
\end{table}

Lastly, we assessed how the agents responded to different natural language inputs at test time.
Using the models presented in the main paper, we evaluated their behavior when using embeddings of ``animal'' or ``vehicle'' instead of embeddings of the class names.

\begin{table}
  \caption{CIFAR10 zero-shot reference game mean reward (standard error) using different NLP inputs for agents trained with 100 supervised examples. Using vehicle/animal labels (``binary'' rows) worsened in-distribution performance but had no or a small positive effect on zero-shot performance.}
  \label{tab:reference_full}
  \centering
  \begin{tabular}{lll}
    \toprule
    Method     &  Label In & Label Out \\
    \midrule
    Cont (BERT) & 91\% (0.7) & 55\% (1.0) \\
    Cont (BERT) - binary &  59\% (0.2)& 56\% (1.0)\\
    \hline
    Proto (BERT) & 93\% (0.7) & 58\% (1.0)\\
    Proto (BERT) - binary & 61\% (0.3) & 60\% (1.0) \\
    \hline
    Proto (Hand) &  75\% (1.1) & 62\% (1.2)\\
    \hline
    One-hot & 96\% (0.2) & 50\% (0.8)\\
    \bottomrule
  \end{tabular}
\end{table}

Results of these experiments are presented in Table~\ref{tab:reference_full}.
Unsurprisingly, in-distribution performance worsened for both \textit{cont} and prototype models.
Not only were embeddings for ``vehicle'' and ``animal'' outside the training distribution, but also such labels were poor choices for distinguishing between two animals, for example.
However, we observed a small positive effect when using the new labels in out-of-distribution evaluation.
This reinforces the conclusion that the \textit{cont} and prototype models learned a communication space that reflected similarities among animals and among vehicles.

\section{Appendix: NLP Embeddings}
\label{app:nlp}
We used a pre-trained BERT-based English sentence embedder developed by huggingface to convert natural language to embeddings \citep{reimers-2019-sentence-bert}.
By default, the embedding dimension is 768; we wished, however, to used a small communication dimension in order to force semantic relationships between the 10 classes.
We therefore extracted the raw embeddings for 68 sentences (combinations of class names and stems (e.g., [``Pick the ", "The "] $\times$ [``cat'', ``dog'', "ship"]) and performed 3D PCA.
Note that PCA was not informed by the classification task, so embeddings for different classes could have been projected into similar spaces.
In many ways, this mimics the dimensionality reduction employed for image feature extraction (explained in Appendix~\ref{app:hyper}).

\begin{figure}
    \centering
    \includegraphics[width=\textwidth]{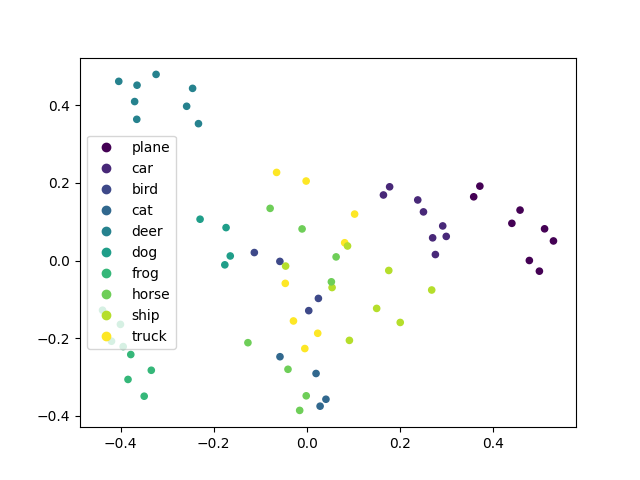}
    \caption{2D PCA of embeddings used to create the low-dimensional language embedder. Even in 2D and without task-specific guidance, language describing the same objects tended to cluster. }
    \label{fig:words_pca}
\end{figure}

In practice, as demonstrated by Figure~\ref{fig:words_pca}, even 2D PCA embeddings of the natural language phrases remained informative of class; suggesting that 3D PCA was at least as informative.
Furthermore, high-level patterns, such embeddings for words for vehicles clustering together (``plane,'' ``car,'', ``ship,'' and ``truck'' on the right of the diagram), also emerged.

Having fit a PCA transform to the data, we saved the transformation function but discarded the 68 sentences that were used in creating the transform.
Instead, when a natural language command was embedded, we started with the natural language, embedded it using our BERT model, and then projected into 3D using the fixed PCA.

\begin{table}[]
    \centering
    \caption{3D vectors used as hand-crafted word embeddings. These embeddings reflect the desired structure of creating clusters by animals and vehicles.}
    \label{tab:hand_embeddings}
    \begin{tabular}{llll}
         Label & x & y & z \\
         \hline
         plane & 0.25 & 0.9 & 0.35 \\
         car & 0.25 & 0.9 & 0.45 \\
         truck & 0.25 & 0.9 & 0.55\\
         ship & 0.25 & 0.9 & 0.65 \\
         bird & 0.75 & 0.05 & 0.2 \\
         frog & 0.75 & 0.25 & 0.2 \\
         cat & 0.75 & 0.05 & 0.3 \\
         dog & 0.75 & 0.25 & 0.3 \\
         deer & 0.75 & 0.05 & 0.4 \\
         horse & 0.75 & 0.25 & 0.4\\
    \end{tabular}
\end{table}

Lastly, in addition to the BERT-based embeddings, we also generated hand-crafted embeddings, specifically designed to encourage agents to learn similarities among animals and among vehicles.
The embedding values are shown in Table~\ref{tab:hand_embeddings}.
These embeddings were certainly not optimal for the reference game, because the proximity of tokens for different classes led to a high rate of confusion due to enviroment noise.
However, by clustering tokens to separate vehicles and animals, this pattern did enable better zero-shot performance than our BERT-based models.

\section{Appendix: User Study Details}
\label{app:amt}
We recruited participants for online studies, conducted on Amazon Mechanical Turk (AMT), to generate descriptions of images or to interpret emergent communication, as explained in Section~\ref{sec:human}.
This experiment was approved by the university Institutional Review Board (IRB).

\begin{figure}
    \centering
    \includegraphics[width=\textwidth]{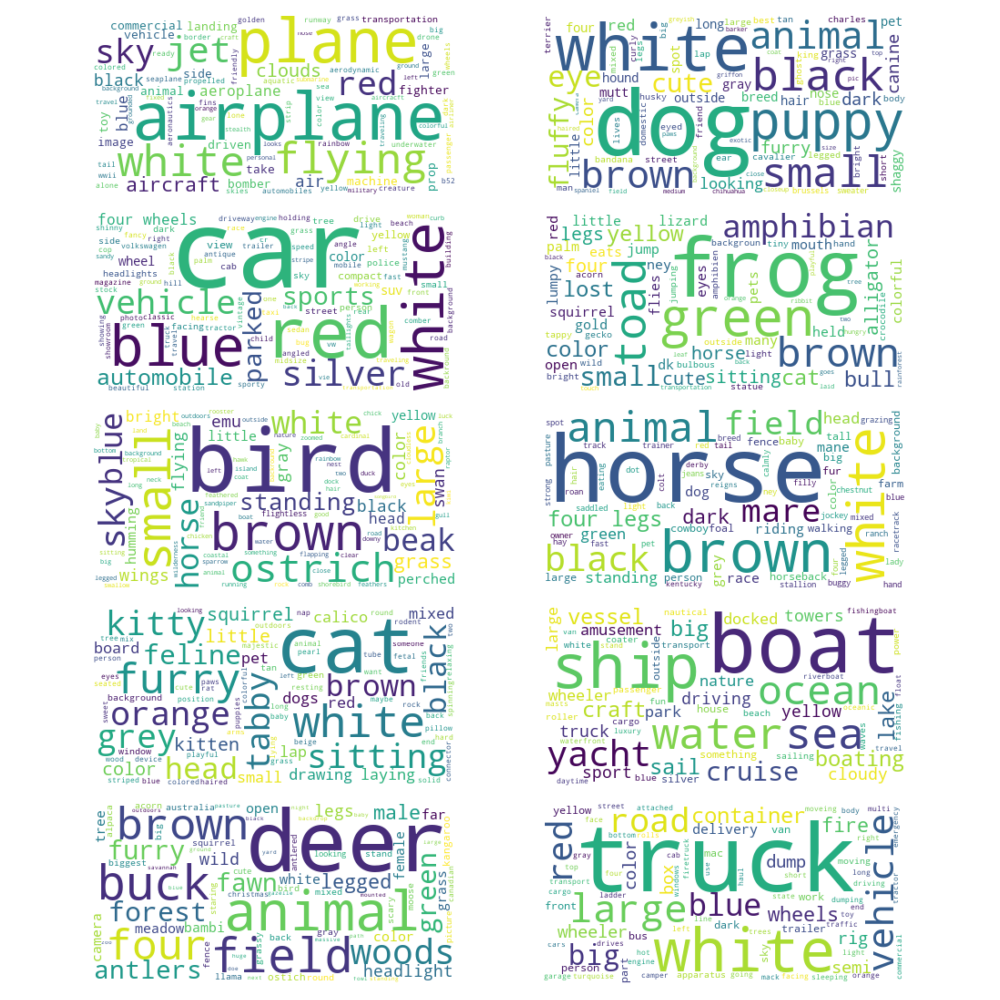}
    \caption{Word cloud visualizations of human-generated annotations for images from the CIFAR-10 dataset.}
    \label{fig:wc}
\end{figure}

In the first data-gathering experiments, we asked participants to generate three labels for images taken from the CIFAR10 dataset. The instruction and user interface of this experiment is shown in Fig.~\ref{fig:amt_1}. We randomly selected 50 images from each of 10 classes and assigned those 500 samples to workers on AMT. In total, 51 unique workers completed the task, resulting in 1257 valid annotations (some images are not distinguishable due to low resolution). Participants were paid $\$0.05$ for each task; the average completion time was 42 seconds, equivalent to a $\$4.28$ hourly wage.

By asking for 3 labels for each image, we hoped to elicit a diverse dataset beyond merely class labels. As shown in the word cloud visualization in Figure~\ref{fig:wc}, annotations generated by MTurkers include both synonyms of class labels (e.g., dog, puppy, canine) and features of specific images (e.g., furry, small, white). In addition, we note that, due to the low resolution of images in the dataset, some participants incorrectly classified images: e.g., labeling a horse as an ostrich.

In the communication interpretation experiment, participants were asked to distinguish the target image from a distractor based on the zero-shot communication sent by agents. In other words, human participants were playing the role of listener in a reference game, teaming up with trained agents. MTurkers were divided into several between-subject groups in which participants received communications from either our prototype-based models or BERT-based embeddings. We tested 5 different trained models using our method and 1 BERT-based embedding on 80 examples of each of 4 animal-vehicle image pairs. After removing tasks that were completed in extremely short or long time (out of 3 standard deviations), we were left with 2093 valid samples from 253 unique workers. Participants were paid $\$0.03$ for each task; the average completion time was $58.9$ seconds, equivalent to a $\$1.84$ hourly wage.

\begin{table}[]
\centering
\begin{tabular}{cccc}
\hline
Group   & Samples & Unique workers & Accuracy \\ \hline
BERT    & 597     & 55             & 65.83\%  \\
Proto-1 & 294     & 36             & 49.32\%  \\
Proto-2 & 301     & 47             & 76.08\%  \\
Proto-3 & 296     & 39             & 70.27\%  \\
Proto-4 & 300     & 38             & 79.00\%  \\
Proto-5 & 305     & 38             & 74.10\%  \\ \hline
\end{tabular}
\caption{Performance table of human communication interpretation experiment.}
\label{tab:AMT}
\end{table}

The instruction and user interface of this experiment is shown in Fig.~\ref{fig:amt_2}. Communications were presented in a 2D plane with 8 labeled nodes and one unlabeled communication node. (Labels were generated by evaluating agents in self-play and selecting the most likely token for each class.) Participants were asked to select one out of two images that the communication node most likely referred to. Both images are from the held-out classes, meaning there were no labeled nodes for the image classes, in order to test zero-shot understanding in human-agent teams. As shown in Table~\ref{tab:AMT}, 4 out of the 5 trained models using our prototype-based method outperformed the BERT-based embeddings. The difference between two methods is marginally significant ($\chi^2(1,N=2093) = 3.21, p = .073$).
Further inspection of Proto-1, the first trained model that exhibits random-chance performance, revealed that the model failed to converge to high reward in training, and that the tokens for vehicles and animals failed to separate.
Training instability is a chronic problem in MARL \citep{NEURIPS2019_fe5e7cb6}, so we consider this failure as a symptom of general difficulties with reinforcement learning rather than our technique specifically.
Nevertheless, in four of our five trials, models did converge to high performance, and if one discards results from Proto-1 as outliers, we significantly outperformed BERT embeddings ($\chi^2(1,N=1799) = 16.15, p < .0001$).

\begin{figure}
    \centering
    \includegraphics[width=\textwidth]{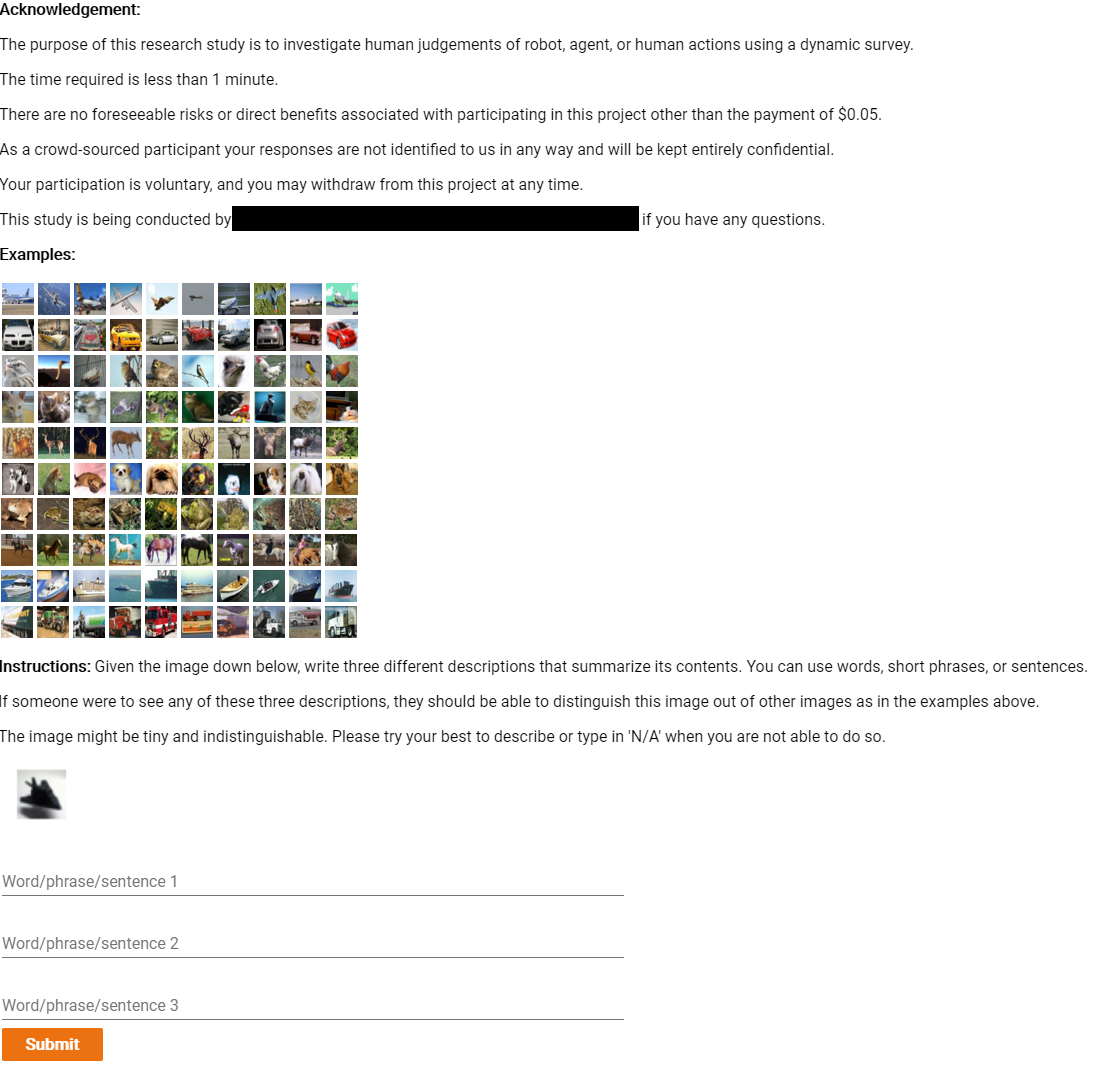}
    \caption{Instructions and user interface of AMT data-gathering experiment.}
    \label{fig:amt_1}
\end{figure}

\begin{figure}
    \centering
    \includegraphics[width=\textwidth]{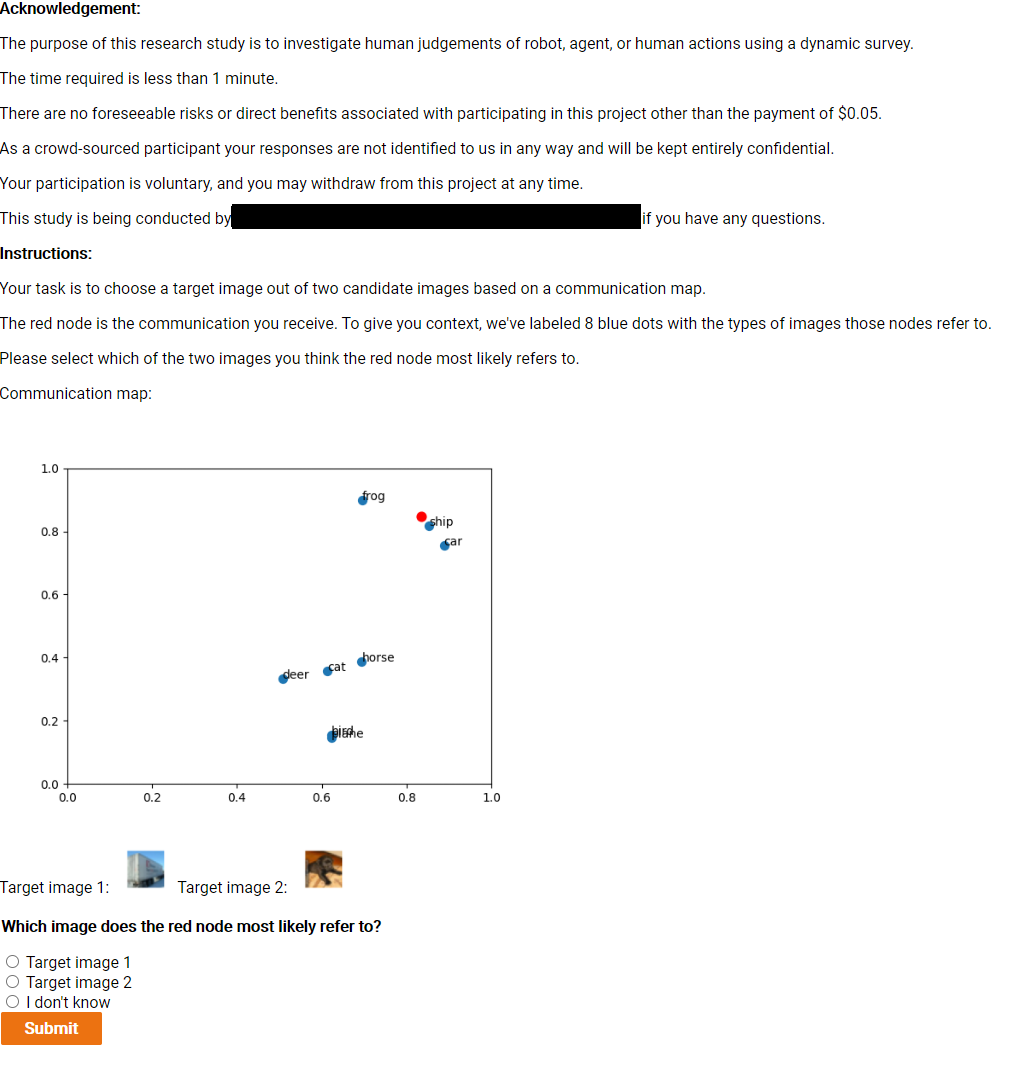}
    \caption{Instructions and user interface of AMT communication interpretation experiment.}
    \label{fig:amt_2}
\end{figure}



\end{document}